\documentclass[preprint,12pt]{elsarticle}

\usepackage{amssymb,amsmath,amsfonts,amsthm}

\newtheorem{thm}{Theorem}
\newtheorem{ass}{Assumption}

\newtheorem{lem}{Lemma}

\usepackage{booktabs}
\usepackage{caption}

\usepackage{booktabs}
\usepackage{caption}
\usepackage{threeparttable}

\journal{Information Fusion}

\begin{document}

\begin{frontmatter}

\title{Soft Label PU Learning}

\author[aff1]{Yong Lei}

\author[aff1]{Puning Zhao}

\author[aff2]{Jintao Deng}

\author[aff3]{Xu Cheng\corref{cor1}}
\ead{chengx19@mails.tsinghua.edu.cn}
\cortext[cor1]{Corresponding author.}

\affiliation[aff1]{organization={Sun Yat-Sen University},
            city={Shenzhen},
            country={China}}

\affiliation[aff2]{organization={Tencent},
            city={Shenzhen},
            country={China}}

\affiliation[aff3]{organization={Tsinghua University},
            city={Beijing},
            country={China}}

\begin{abstract}
PU learning refers to the classification problem in which only part of positive samples are labeled. Existing PU learning methods treat unlabeled samples equally. However, in many real tasks, from common sense or domain knowledge, some unlabeled samples are more likely to be positive than others. In this paper, we propose \emph{soft label PU learning}, in which unlabeled data are assigned soft labels according to their probabilities of being positive. Considering that the ground truth of $TPR$, $FPR$, and $AUC$ are unknown, we then design PU counterparts of these metrics to evaluate the performances of soft label PU learning methods within validation data. We show that these new designed PU metrics are good substitutes for the real metrics. After that, a method that optimizes such metrics is proposed. Experiments on public datasets and real datasets for anti-cheat services from Tencent games demonstrate the effectiveness of our proposed method.
\end{abstract}


\end{frontmatter}

\section{Introduction}

Positive-unlabeled (PU) learning addresses binary classification when only a subset of positive examples is labeled and the remaining data are left unlabeled. This setting is common in applications where positive events can be partially confirmed but reliable negative labels are difficult, delayed, or costly to obtain. Under the classical SCAR assumption, PU learning has developed a solid methodological foundation, with most existing methods treating unlabeled examples uniformly as equally ambiguous~\citep{elkan2008learning}.

However, real PU data often deviate from the homogeneous SCAR setting. In practice, the labeling mechanism can be selection-biased or instance-dependent \citep{bekker2019beyond,gong2025surveyidpu}, and the distributions of labeled positives and latent positives in the unlabeled pool may differ under distribution shift \citep{liu2025detrpu}. In addition, recent studies have highlighted that model selection and evaluation in PU learning remain difficult in realistic settings, because validation should not rely on access to true negative labels \citep{wang2025fairpu}.

These limitations motivate a simple but important question: how should we learn and validate a PU classifier when unlabeled instances are accompanied by sample-wise soft evidence about their positivity? Such a setting arises naturally in many real applications. For example, an undiagnosed patient may exhibit symptoms with different levels of severity; an account in an anti-cheat system may show different degrees of suspicious behavior; and an exposed-but-unconfirmed user in advertising or recommendation may carry different levels of conversion intent. In all these cases, treating all unlabeled examples equally discards useful information, whereas directly treating such signals as hard labels is also inappropriate.

In this paper, we formulate this setting as \emph{soft label PU learning}. Each unlabeled sample is associated with a soft label in $[0,1]$ that represents prior evidence about its probability of being positive. Our goal is not only to exploit these soft labels during training, but also to address a more fundamental issue: in the absence of true validation labels, how can we assess whether a model is improving in the right direction? To this end, we introduce PU counterparts of standard binary classification metrics, including $TPR_{SPU}$, $FPR_{SPU}$, and $AUC_{SPU}$, and study when these surrogate metrics faithfully reflect the corresponding fully supervised metrics.

Our analysis shows that these soft-label-aware PU metrics provide a principled basis for surrogate validation under a hierarchy of assumptions. Under a generalized SCAR condition, the proposed PU metrics are linearly related to the true metrics, so improvements in the surrogate metrics imply improvements in the real ones. Under weaker monotonicity conditions, the optimal classifier induced by the PU metrics remains aligned with the optimal classifier under true labels. Even when monotonicity is only approximately satisfied, the resulting performance gap can still be controlled. These results suggest that soft evidence can be incorporated into PU learning in a theoretically meaningful and practically useful way.

Motivated by this analysis, we further propose a simple learning strategy that estimates the conditional expectation of the soft label and induces classifiers that optimize the soft-label-aware PU criteria. Experiments on tabular benchmarks, image datasets, and a real anti-cheat task from Tencent Games demonstrate that incorporating sample-wise soft evidence can substantially improve PU classification compared with treating all unlabeled samples uniformly.

The main contributions of this paper are as follows:
\begin{itemize}
   
    \item We propose a new family of soft-label-aware PU evaluation metrics, including $TPR_{\mathrm{SPU}}$, $FPR_{\mathrm{SPU}}$, and $AUC_{\mathrm{SPU}}$, and further characterize the upper bound of $AUC_{\mathrm{SPU}}$.

    \item We provide theoretical analysis of when these surrogate metrics are reliable, showing their applicability under different assumptions and clarifying when they can serve as faithful proxies for the underlying fully supervised criteria.
    
    \item We develop a practical learning method motivated by the proposed metrics and validate its effectiveness on both public benchmarks and a real-world anti-cheat application.
\end{itemize}

\section{Related Work}

\medskip
\noindent\textbf{Classical PU learning under SCAR.}
\smallskip
PU learning has been extensively studied in the setting where labeled positives are assumed to be selected completely at random from the positive class. Early work showed that a classifier can be trained by treating unlabeled samples as negatives followed by suitable score correction~\citep{elkan2008learning}. Risk-consistent approaches were later developed to estimate the supervised risk from positive and unlabeled data, including the unbiased PU risk estimator~\citep{du2014analysis} and the non-negative PU estimator~\citep{kiryo2017positive}. Beyond these foundational methods, a variety of extensions have been proposed, such as generative, variational, and self-training based approaches~\citep{hou2017generative, chen2020variational, chen2020self}. These methods are effective when SCAR approximately holds and when the class prior can be estimated or is known.

\medskip
\noindent\textbf{Beyond SCAR: realistic PU settings.}
\smallskip
A major limitation of classical PU learning is that the labeling mechanism is often not random in practice. This has motivated a large body of work on more realistic PU settings, where the probability of observing a positive label depends on the instance itself. Representative examples include probabilistic-gap based PU learning~\citep{he2018instance} and feature-dependent labeling models trained by expectation-maximization~\citep{bekker2019beyond}. More recently, instance-dependent PU learning has been systematically reviewed as an important research direction, emphasizing that real-world PU data are often generated by confidence-dependent or mechanism-dependent selection rather than SCAR~\citep{gong2025surveyidpu}. Statistical formulations have also become more expressive: recent work studies positive and unlabeled data under transfer and distribution shift using identifiable density-ratio models~\citep{liu2025detrpu, sang2025gadrpu}, and very recent results further consider prior-shift estimation for PU target data~\citep{mielniczuk2026priorshift}. These studies show that modern PU learning increasingly focuses on realism, uncertainty, and distribution mismatch rather than relying solely on the classical homogeneous unlabeled assumption.

\medskip
\noindent\textbf{Evaluation and validation in PU learning.}
\smallskip
While classifier training has received most of the attention in PU learning, validation and evaluation without negative labels are equally important and remain underexplored. Recent work has explicitly pointed out that many PU studies still rely on unrealistic validation protocols that assume access to negatives, and has called for fairer and more realistic evaluation benchmarks~\citep{wang2026fairpu}. Relatedly, calibration under PU supervision has only recently begun to be studied, with PU-specific expected calibration error estimators proposed for reliability assessment without fully labeled validation data~\citep{kiryo2025puece}. Our work is closely related to this emerging line, but addresses a different gap: instead of studying calibration of prediction confidence under ordinary PU supervision, we consider how to define and analyze \emph{classification metrics themselves} when each unlabeled instance carries a sample-wise soft hint about positivity.

\medskip
\noindent\textbf{Our position.}
\smallskip
In contrast to existing PU methods that either assume homogeneous unlabeled data or focus primarily on training objectives, we study a complementary setting in which unlabeled samples are accompanied by \emph{soft evidence} derived from prior knowledge. This makes our problem closer to a bridge between PU learning and weak supervision with probabilistic hints. The key question is therefore not only how to train with such evidence, but also how to obtain reliable validation and evaluation signals when ground-truth validation labels are unavailable. Our proposed soft-label-aware PU metrics are designed precisely for this purpose.

\section{Problem Statement}
Suppose that we have $N$ samples $\left(X_{i}, Y_{i}, S_{i}\right), i=1, \ldots, N$. These $N$ samples are independent and identically distributed (i.i.d), according to a common distribution $\mathbb{P}_{X, Y, S} . X_{i} \in \mathbb{R}^{d}$ denotes the feature, $Y_{i} \in\{0,1\}$ denotes the true label, which is unknown to us. $S_{i} \in[0,1]$ is the soft label, which is calculated from prior knowledge. If a sample $i$ has $S_{i}=0$, then it is an ordinary unlabeled sample. If $S_{i}=1$, then it is a positive example. If $S_{i} \in(0,1)$, then it is still an unlabeled sample, but is more likely to be positive than ordinary unlabeled samples. $S_{i}$ is known from our prior knowledge of the dataset. We make the following basic assumption:

\begin{ass}\label{ass:basic}
\begin{equation}
	\mathbb{E}[S \mid Y=1]>\mathbb{E}[S \mid Y=0].
\end{equation}	
\end{ass}

Assumption 1 holds as long as the assigned soft label $S$ is positively correlated to the real label $Y$. If $S$ can only take values in $\{0,1\}$, then the problem reduces to the traditional PU learning problem. Our task is to learn a classifier $\hat{Y}(X) \in\{0,1\}$, to predict the true label $Y$. 

Before designing PU metrics, we provide a brief review of common evaluation metrics for classification, including $TPR$, $FPR$ and $AUC$:
\begin{eqnarray}
	T P R&=&\mathbb{P}(\hat{Y}=1 \mid Y=1),\\
	F P R&=&\mathbb{P}(\hat{Y}=1 \mid Y=0)
\end{eqnarray}
Moreover, many classifiers actually output a score $g(X)$ and the final prediction is made by checking whether the score is larger than a certain threshold, i.e.

\begin{equation}
	\hat{Y}=I(g(X)>T), \nonumber
\end{equation}

in which $T$ is the threshold, $I$ is indicator function, and $g_{s}$ is the raw classifier that outputs scores. $TPR$ and $FPR$ change with $T$. To evaluate the overall performance of a classifier over all possible thresholds, define
\begin{equation}
	A U C=\int_{0}^{1} T P R d(F P R).
\end{equation}

In traditional machine learning problems, these metrics can be measured using validation data, which is taken randomly
from training dataset. However, in the soft label PU learning problem, $Y_{i}$ is unknown. Therefore, we need to design a substitute of the metrics, to serve as the objectives to optimize. Here we design such metric by constructing a similar counterpart of $T P R, F P R$ and $A U C$ based only on soft label $S$ and prediction $\hat{Y}$. To begin with, note that

\begin{equation}
		T P R  =\mathbb{P}(\hat{Y}=1 \mid Y=1) =\frac{\mathbb{P}(\hat{Y}=1, Y=1)}{\mathbb{P}(Y=1)}  =\frac{\mathbb{E}[\hat{Y} Y]}{\mathbb{E}[Y]}.
\end{equation}

We can similarly define the substitute $T P R_{S P U}$ to be

\begin{equation}
	T P R_{S P U}:=\frac{\mathbb{E}[S \hat{Y}]}{\mathbb{E}[S]}
	\label{eq:tprdef}
\end{equation}

Similarly, we can define

\begin{equation}
	F P R_{S P U}:=\frac{\mathbb{E}[(1-S) \hat{Y}]}{\mathbb{E}[1-S]}
	\label{eq:fprdef}
\end{equation}

\eqref{eq:tprdef} and \eqref{eq:fprdef} define true/false positive rate in the setting of soft label PU learning. Since the distribution is unknown, $TPR_{SPU}$ and $FPR_{SPU}$ cannot be calculated exactly. Instead, these two quantities can be estimated using the validation dataset:
\begin{eqnarray}
 \widehat{T P R}_{S P U}&=&\frac{\sum_{i=1}^{N_{V}} S_{i} \hat{Y}_{i}}{\sum_{i=1}^{N_{V}} S_{i}}   \\
\widehat{F P R}_{S P U}&=&\frac{\sum_{i=1}^{N_{V}}\left(1-S_{i}\right) \hat{Y}_{i}}{\sum_{i=1}^{N_{V}}\left(1-S_{i}\right)},  	
\end{eqnarray}

in which $N_{V}$ is the number of validation samples.

Following the definition of $R O C$ and $A U C$ in traditional machine learning, we now define $R O C_{S P U}$ as the curve of $TPR_{SPU}$ versus $FPR_{SPU}$. Correspondingly, define $A U C_{S P U}$ as the area under the $ROC_{SPU}$ curve, i.e.

\begin{equation}
	A U C_{S P U}:=\int T P R_{S P U} d\left(F P R_{S P U}\right) .
\end{equation}

We then discuss the range of $A U C_{S P U}$, which is shown in the following theorem.

\begin{thm}\label{thm:auc}
	Given the distribution of $S$, the value of $A U C_{S P U}$ satisfies
	\begin{equation}
		A U C_{S P U} \leq \frac{1}{2}+\frac{\int_{0}^{1} F_{S}(u)\left(1-F_{S}(u)\right) d u}{2 \int_{0}^{1} F_{S}(u) d u \int_{0}^{1}\left(1-F_{S}(u)\right) d u},
		\label{eq:aucbound}
	\end{equation}
	in which $F_{S}$ is the cumulative distribution function (cdf) of soft label $S$.
\end{thm}

It can be shown that the right hand side of \eqref{eq:aucbound} is 1 if and only if $\mathbb{P}(S \in(0,1))=0$, which means that with probability 1 , $S$ can only take values in $\{0,1\}$. In this case, the problem reduces to traditional PU learning problem. Otherwise, if $\mathbb{P}(S \in(0,1))>0$, then the right hand side of \eqref{eq:aucbound} is always less than 1.

\section{Analysis of PU metrics}
In section 3, we have defined metrics for soft PU learning, which can be estimated using validation data. We will then select model and adjust parameters with the goal of maximizing $T P R_{S P U}$ and minimizing $F P R_{S P U}$ simultaneously. Then a natural question occurs: Are these metrics good substitute of the real underlying metric $T P R, F P R$ and $AUC$? In this section, we answer this question by comparing the PU substitute metrics with real metrics. In other words, we would like to check whether improving $T P R_{S P U}$ and $F P R_{S P U}$ implies improvement of $T P R$ and $FPR$. Throughout this section, we define

\begin{equation}
	\pi=\mathbb{P}(Y=1)
\end{equation}

as the probability of a sample to be positive, and
\begin{eqnarray}
	 S_P&=&\mathbb{E}[S \mid Y=1], \\
S_N&=&\mathbb{E}[S \mid Y=0].
\end{eqnarray}

In this section, we will show that even though the prior knowledge is not very exact, and those soft labels are not very carefully designed, as long as Assumption 1 is satisfied, i.e. $S_{P}>S_{N}$, then under various assumptions, we can get a desirable performance by optimizing those PU substitute metrics.

\subsection{Under Generalized SCAR Assumption}
In original PU learning problems, SCAR assumption assumes that positive samples are labeled with equal probability, and thus the distribution of observed positive samples equals the distribution of all positive samples, i.e. $f(x \mid S=1)=f(x \mid Y=1)$, in which $f$ is the probability density function (pdf) of features. Since we now have soft labels for unlabeled data, we can generalize the above assumption and design a Generalized SCAR assumption, which is formulated as following.

\begin{ass}\label{ass:scar}
(Generalized SCAR assumption) $X$ and $S$ are conditional independent given $Y$.	
\end{ass}

Assumption 2 indicates that for any $s$, we have
\begin{eqnarray}
	f(x \mid S=s, Y=1)&=&f(x \mid Y=1),\label{eq:pos}\\
	f(x \mid S=s, Y=0)&=&f(x \mid Y=0).\label{eq:neg}
\end{eqnarray}
Obviously, the generalized SCAR assumption can be reduced to SCAR assumption, in the sense that if $S$ can only take values in $\{0,1\}$, then \eqref{eq:pos} and \eqref{eq:neg} are equivalent to the original SCAR assumption. In this case, we show the following theoretical results.

\begin{thm}\label{thm:linear}
	If the distribution $\mathbb{P}_{X, Y, S}$ satisfies Generalized SCAR assumption (Assumption \ref{ass:scar}), then the following results hold:
	\begin{eqnarray}
	T P R_{S P U} & =&a T P R+b F P R, \label{eq:tprspu}\\
F P R_{S P U} & =&c T P R+d F P R,\label{eq:fprspu} \\
A U C_{S P U} & =&\frac{1}{2}(b+c)+(a d-b c) A U C,\label{eq:aucspu}		
	\end{eqnarray}
in which
\begin{eqnarray}
	a & =&\frac{\pi S_{P}}{\pi S_{P}+(1-\pi) S_{N}},\label{eq:a} \\
b & =&\frac{(1-\pi) S_{N}}{\pi S_{P}+(1-\pi) S_{N}}, \\
c & =&\frac{\pi\left(1-S_{P}\right)}{1-\pi S_{P}-(1-\pi) S_{N}}, \\
d & =&\frac{(1-\pi)\left(1-S_{N}\right)}{1-\pi S_{P}-(1-\pi) S_{N}}.\label{eq:d}	
\end{eqnarray}
Furthermore, as long as $S_{P}>S_{N}$, we have
\begin{equation}
	a d-b c>0.
\end{equation}	
\end{thm}

Note that $\pi, S_{P}$ and $S_{N}$ are usually unknown to us. Therefore, even if we have obtained linear relationship between the PU substitute metrics, they cannot be used to infer the value of real metrics. However, from (17) to (19), we know that the direction of improvement pointed out by these PU metrics is correct, since if a new algorithm improves $A U C_{S P U}$, then the underlying $A U C$ will also improve. When optimal $R O C_{S P U}$ is reached, optimal $R O C$ is also reached.

The analysis above indicates that under Generalized SCAR assumption, the PU metrics $T P R_{S P U}$, $F P R_{S P U}$ and $A U C_{S P U}$ can substitute the corresponding real metrics well. However, in some cases, the generalized SCAR assumption (Assumption \ref{ass:scar}) is not satisfied. Therefore, we analyze the PU metric under weaker assumptions.

\subsection{Under Monotonic Expected Label Assumption}
Here we use an assumption that is weaker than generalized SCAR assumption.

\begin{ass}\label{ass:mela}
	(Monotonic Expected Label Assumption) There exists a monotonic increasing function $h$, such that
	
	\begin{equation}
		\mathbb{E}[S \mid X]=h(\mathbb{P}(Y=1 \mid X)).
	\end{equation}
\end{ass}

An intuitive understanding of Assumption \ref{ass:mela} is that if the positive conditional probability is high under some feature values, then the expected soft label will also be higher.

It can be shown that Assumption \ref{ass:mela} is weaker than the combination of Assumption \ref{ass:scar} and the basic setting $S_{P}>S_{N}$ :
\begin{eqnarray}
		\mathbb{E}[S \mid X]&= &\mathbb{P}(Y=1 \mid X) \mathbb{E}[S \mid Y=1, X] +\mathbb{P}(Y=0 \mid X) \mathbb{E}[S \mid Y=0, X] \nonumber\\
&= & \mathbb{P}(Y=1 \mid X) \mathbb{E}[S \mid Y=1] +P(Y=0 \mid X) \mathbb{E}[S \mid Y=0] \nonumber\\
&= & \left(S_{P}-S_{N}\right) P(Y=1 \mid X)+S_{N},	
\end{eqnarray}

in which the second step holds because Assumption 2 assumes the conditional independence of $S$ and $X$ given $Y$. As long as $S_{P}>S_{N}$, (25) is satisfied with $h(t)=$ $\left(S_{P}-S_{N}\right) t+S_{N}$. This shows that Assumption 3 is weaker. Actually, comparing with Assumption 2, Assumption 3 is more realistic since it is usually hard to design soft labels to ensure that the conditional expectation of soft label grows linearly with $\mathbb{P}(Y=1 \mid X)$.

Under this assumption, we can show the following results:
\begin{thm}\label{thm:mela}
Under Assumption \ref{ass:mela}, if a classifier $g_{s}^{*}: \mathbb{R}^{d} \rightarrow$ $\mathbb{R}$ attains optimal $R O C_{S P U}$, then $g_{s}^{*}$ attains optimal $R O C$.	
\end{thm}

Theorem 3 indicates that even if SCAR assumption is not satisfied, as long as the conditional expectation of observed label grows strictly with the regression function $\eta(x)=\mathbb{E}[Y \mid X]$, then the PU substitute metrics are still reliable since when $R O C_{S P U}$ and $A U C_{S P U}$ are optimal, $R O C$ and $A U C$ are also optimal. The difference with the case under Generalized SCAR assumption is that under this new assumption, an improvement of $A U C_{S P U}$ does not necessarily indicate an improvement of $A U C$.

\subsection{Under Noisy Monotonic Expected Label Assumption}
Assumption 3 is still not satisfied by some cases, in which at the locations with the same $\mathbb{P}(Y=1 \mid x), \mathbb{E}[S \mid X]$ are still different. To cope with these cases, we weaken the assumption further.

\begin{ass}\label{ass:nmela}
(Noisy Monotonic Expected Label Assumption) There exists a monotonic increasing function $h$, such that

(a) $h^{\prime} \geq C_{h}$ for some constant $C_{h}$;

(b) $|E[S \mid X]-h(\mathbb{P}(Y=1 \mid X))| \leq \epsilon$.
\end{ass}

Under Assumption 4, $\mathbb{E}[S \mid X]$ roughly grow with $\mathbb{E}[Y \mid X]$, with some fluctuations allowed. Under this assumption, we can show the following results.
\begin{thm}\label{thm:nmela}
 Under Assumption 4, assume that there exists a constant $M$, such that for any $a, b$ such that $0<a<b<1$,

\begin{equation}
	\int_{\mathbb{E}[Y \mid X] \in[a, b]} f(x) d x \leq M(b-a)
	\label{eq:slice}
\end{equation}

Then (1) If the TPR and FPR of a classifier $\hat{Y}$ is on the ROC curve, then there exists a classifier $\hat{Y}_{S}$ on the $R O C_{S P U}$ curve, such that
\begin{eqnarray}
 T P R(\hat{Y})-T P R\left(\hat{Y}_{S}\right) &\leq& \frac{4 M}{\pi C_{h}^{2}} \epsilon^{2}, \\
 FPR\left(\hat{Y}_{S}\right)-FPR(\hat{Y}) &\leq& \frac{4 M}{\pi C_{h}^{2}} \epsilon^{2} ;	
\end{eqnarray}

(2) If a scoring function $g_{s}^{*}: \mathbb{R}^{d} \rightarrow \mathbb{R}$ attains optimal $R O C_{S P U}$, then

\begin{equation}
	\max _{g} A U C(g)-A U C\left(g_{s}^{*}\right) \leq \frac{8 M}{\pi C_{h}^{2}} \epsilon^{2} .
\end{equation}	
\end{thm}

Figure 1 illustrates Theorem 4. If the conditional expectation of soft label does not monotonically grow with $\mathbb{P}(Y=1 \mid X)$, the $R O C$ obtained by optimizing the PU metric $R O C_{S P U}$ will no longer be exactly the same as the optimal $R O C$ when true labels are completely known. However, the deviation is only $O\left(\epsilon^{2}\right)$, which is second order small. Therefore, as long as the deviation of such monotonicity is not severe, $T P R_{S P U}, F P R_{S P U}$ and $A U C_{S P U}$ can still work as good metrics.
\begin{figure}[!t] %
	\centering
	\includegraphics[width=0.5\linewidth]{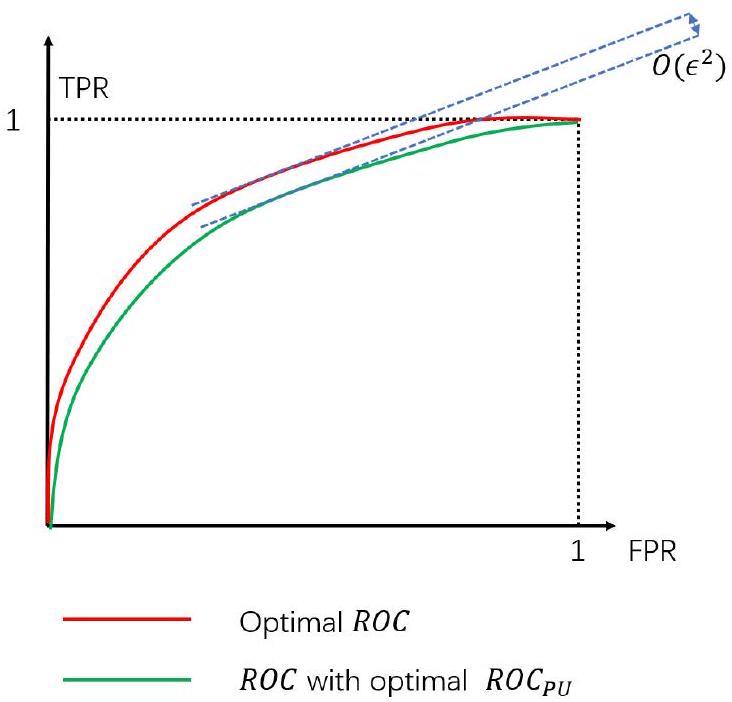} 
	\caption{Illustration of the error of the ROC curve obtained by optimizing PU metrics.}
	\label{fg:1}
\end{figure} 
%

Now we summarize our findings. We have analyzed the performance of $T P R_{S P U}, F P R_{S P U}$ and $A U C_{S P U}$ under three assumptions, in which the first one is the strongest and the last one is the weakest. If Generalized SCAR assumption is satisfied, then these PU metrics can perfectly replace real metrics $T P R, F P R$ and $A U C$, since every improvement on PU metrics implies the improvement on real metrics. If the assumption is weakened, such that we only require $\mathbb{E}[S \mid X]$ to grow monotonically with $\mathbb{E}[Y \mid X]$, then the improvement on PU metrics does not necessarily imply the improvement on real metrics. However, if PU metrics are optimal, then real metrics are also optimal. Then we further weakened the assumption, such that $\mathbb{E}[S \mid X]$ is not required to grow completely monotonic with $\mathbb{E}[Y \mid X]$, a noise with level $\epsilon$ is allowed. In this case, the optimality of PU metrics does not imply the optimality of real metrics, but the gap is only second order small. Therefore, we claim that the PU metrics can, in general, provide us with the correct direction to help us design and improve our model.

\section{Learning Methods}
In section 4, we have shown that after trying to improve these substitute metrics, the distance between our final classifier and the ideal classifier is only $O\left(\epsilon^{2}\right)$. As a result, we can just focus on improving these PU metrics. In this section, we discuss how to optimize these metrics given a real dataset.

To begin with, we assume that we have an infinite amount of data. This implies that $\mathbb{E}[S \mid X=x]$ is known for every $x$. In this case, we show the following theorem
\begin{thm}\label{thm:roccurve}
The $T P R_{S P U}$ and $FPR_{S P U}$ of classifier
\begin{equation}
	g^{*}(X)=I(\mathbb{E}[S \mid X]>T)
\end{equation}
are on the optimal $R O C_{S P U}$ curve.	
\end{thm}

This theorem tells us that if the sample size is infinite, such that $\mathbb{E}[S \mid X]$ can be accurately estimated, we can just classify a new sample by checking whether the conditional expectation of soft label is larger than a threshold. Of course, in reality, the amount of data is limited. In this case, we can design a method to approximate $g^{*}(X)$ as close as possible. With this intuition, we propose our method as follows: Given a parametric model $g_{s}(x, w)$, e.g. a deep neural network, we minimize the following empirical loss:

\begin{equation}
		L_{e m p}(w)=  -\frac{1}{N} \sum_{i=1}^{N}\left[s_{i} \ln g_{s}\left(x_{i}, w\right)+\left(1-s_{i}\right) \ln \left(1-g_{s}\left(x_{i}, w\right)\right)\right]
\end{equation}

Let $w_{e m p}^{*}$ be the minimizer of $L_{e m p}(w)$. Now we briefly show that $g\left(x, w_{e m p}^{*}\right)$ will converge to the optimal classifier with the growth of sample size $N$ and the model complexity. Define
\begin{eqnarray}
		L(w) =  -\mathbb{E}\left[S \ln g_{s}(X, w)+(1-S) \ln \left(1-g_{s}(X, w)\right)\right]	
\end{eqnarray}
and
\begin{equation}
	w^{*}=\arg \min L(w).
\end{equation}

Then previous results on the convergence of neural networks ~\citep{bruck1988generalized, kohler2019rate, kohler2021rate, muller2012neural} still hold here, and it can be shown that with the growth of the sample size $N$,

\begin{equation}
	\lim _{N \rightarrow \infty} \sup _{x}\left|g_{s}\left(x, w_{e m p}^{*}\right)-g_{s}\left(x, w^{*}\right)\right|=0 .
\end{equation}

Moreover, it can be easily shown from (31) that given the joint distribution of $S$ and $X, L(w)$ is minimized if $g(X, w)=\mathbb{E}[S \mid X]$. From universal approximation theorem, with the growth of complexity, $g_{s}\left(x, w^{*}\right)$ will gradually be closer to $\mathbb{E}[S \mid X]$. Combine this fact with (33), we know that the output score of the model trained by minimizing (32) converges to $\mathbb{E}[S \mid X]$. From Theorem 5, we know that the $T P R_{S P U}, F P R_{S P U}$ of the final classifier with thresholding will converge to a point on the optimal $R O C_{S P U}$ curve. Note that it is also possible to use nonparametric classifiers, such as $k$ nearest neighbor classifier \citep{liu2006new,zhao2021minimax,zhao2021efficient}. 

The analysis above shows that the classifier obtained by minimizing (33) converges to the optimal classifier under these PU metrics, including $T P R_{S P U}$, 
$F P R_{S P U}$ and $A U C_{S P U}$. In the next section, we will test the performance of our new method on some real datasets.

\section{Experiments}
In this section, we show experiments that validate our soft label PU learning method. Our experiments are conducted on three types of datasets. The first one comes from UCI repository. The second type includes two common image datasets, Fashion MNIST and CIFAR-10. For the first and the second type, since we do not assume the knowledge of class prior, we use $A U C$ instead of accuracy as the metric. Finally, we show the experiment result in anti-cheat task in Tencent Games, which is a realistic task.

\subsection{Results on UCI repository}
We use three datasets: Diabetes~\citep{strack2014impact}, Adult and Breast Cancer. These datasets are all fully labeled. However, in training datasets, we assume that the labels are not completely known, and only part of positive samples are labeled. The labeling mechanism is determined in the following way. To begin with, we generate a new feature in the training data, whose values follow a uniform distribution in $[0,0.5]$. Then for each training sample, the probability of being labeled equals the value of this new feature.

Our soft label PU learning method is designed as follows. The first step is to identify some features from the dataset that are expected to be related to the probability of being positive. Note that these features are selected from common sense instead of being selected from data. Based on these features, we then generate soft labels.

For the Diabetes dataset, the task is to predict the patients that possibly suffer from diabetes. We expect that the people who visited hospital many times or have received many medical treatments are more likely to be positive than ordinary unlabeled samples. Therefore, we assign soft labels based on four features, i.e. number\_medications, number\_outpatient, number\_emergency and number\_inpatient. For the Adult dataset, we need to predict whether a person earns over $50 \mathrm{~K}$ a year. We expect that people who have some investments are more likely to earn over $50 \mathrm{~K}$, thus capital-gain and capital-loss are two features used to generate soft labels. For the breast cancer dataset, we use worst-radius and worsttexture for soft labels. Actually, in all these three datasets, there may exist better alternative methods to generate soft labels based on the opinions from domain experts.

We use two efficient models to train classifiers based on the generated soft labels, i.e. XGBoost ~\citep{chen2016xgboost} and LightGBM ~\citep{ke2017lightgbm} for classification. To avoid overfitting, we remove the features used to generate soft labels before feeding the training data into models.

To show the effectiveness of the soft label PU learning method, we design a contrast method for baseline, in which we just train a classifier using XGBoost and LightGBM with the labeled positive samples being regarded as positive samples, and those unlabeled as negative samples. Another difference with the soft label PU learning method is that in soft label method, the features used to generate soft labels are removed, while in the baseline method, these features are incorporated in model training.

The AUCs are then compared between the soft label PU learning method and the baseline. The results are shown as follows, in which the AUCs are evaluated using test data with true labels.

\begin{table}[t]
\centering
\begin{threeparttable}
\caption{AUC comparison between the soft-label PU learning method and the baseline.}
\label{tab:auc_uci}
\small
\begin{tabular*}{0.85\linewidth}{@{\extracolsep{\fill}}lccc@{}}
\toprule
DATA SET & MODEL & RESULT & BASELINE \\
\midrule
DIABETES      & XGB  & 0.721 & 0.694 \\
DIABETES      & LGBM & 0.728 & 0.700 \\
ADULT         & XGB  & 0.834 & 0.829 \\
ADULT         & LGBM & 0.863 & 0.833 \\
BREAST CANCER & XGB  & 0.934 & 0.885 \\
BREAST CANCER & LGBM & 0.910 & 0.876 \\
\bottomrule
\end{tabular*}
\end{threeparttable}
\end{table}

The result shows that for all cases, our soft label PU learning methods perform better than simple PU learning baselines, despite that the generation mechanism for soft labels are actually not very carefully designed. A simple explanation is that since the probability of a positive sample to be labeled
depends on the new feature we generated, the labeling mechanism does not satisfy SCAR assumption. In this case, there are some samples that are actually positive, but the labeling mechanism makes them hard to be discovered. With our new method, these hidden positive samples can possibly be assigned a soft label, thus their chances of being identified by classifiers can be greatly enhanced.

\subsection{Results on image datasets}
We then test the performance of soft label PU learning in two well known image datasets, i.e. Fashion-MNIST and CIFAR10. For these two datasets, we generate soft label under Generalized SCAR assumption, i.e. Assumption 2. To be more precise, for each sample with $Y_{i}=1$, we let

\begin{equation}
	\mathbb{P}\left(S_{i}=\frac{k}{4}\right)=\frac{1}{5}
\end{equation}

for $k=0, \ldots, 4$, which means that each real positive example is labeled with $0,0.25,0.5,0.75$ and 1 with equal probabilities. For each sample with $Y_{i}=0$, we let

\begin{equation}
	\mathbb{P}\left(S_{i}=\frac{k}{4}\right)=\frac{\pi(4-k)}{5 k(1-\pi)}
\end{equation}

for $k=1, \ldots, 3$, in which $\pi$ is the ratio of positive samples. It can be proved that by Bayes theorem, we have

\begin{equation}
	\mathbb{P}(Y=1 \mid S=s)=s
\end{equation}

for $s=0.25,0.5,0.75$. In this setting, the soft label has the meaning of the probability of a sample with unknown label to be positive.

With this setting, we use a simple convolutional neural network to classify images. Since both two datasets have 10 classes, we select some of classes as positive, while the others are regarded as negative samples. We design two baselines. The first baseline regards all samples with soft labels in $(0,1)$ as ordinary unlabeled samples, while the second baseline regards them as positive samples. The AUC results are shown in the following table, in which BLO and BL1 stand for the first and second baselines, respectively.

\begin{table}[htbp]
\centering
\begin{threeparttable}
\caption{AUC comparison between the soft-label PU learning method and the baselines on image datasets.}
\label{tab:image_auc}
\small
\begin{tabular*}{0.92\linewidth}{@{\extracolsep{\fill}}lcccc@{}}
\toprule
Dataset & Positive class & Result & BLO & BL1 \\
\midrule
F-MNIST  & 0           & 0.980 & 0.958 & 0.978 \\
F-MNIST  & $0,2,4,6$   & 0.994 & 0.983 & 0.993 \\
CIFAR 10 & 0           & 0.930 & 0.898 & 0.900 \\
CIFAR 10 & $0,1,8,9$   & 0.951 & 0.863 & 0.936 \\
\bottomrule
\end{tabular*}
\end{threeparttable}
\end{table}

It can be observed that the soft label PU learning shows the best performance for all cases.

\subsection{Results on anti-cheat services in Tencent Games}
Finally, we evaluate our method on the anti-cheat tasks in Tencent Games, which is a realistic application. 
Online video games suffer from all kinds of cheating including game-hacking scripts, fake adult credentials and etc, which hurt user experience severely and lead to a tremendous challenge in terms of legal risk prevention. We have already enacted multiple rules to identify these potential cheaters, and then require them to pass a security check. After that, these users can be divided into three groups based on the results. The first group admits that they are cheaters, mainly because security checks are randomly assigned in terms of time and majority of cheaters are not able to pass the sudden check alone. As a result, accounts in the first group will suffer from punishment, including time restriction, account bans and etc. The second group of users pass the checking process successfully. They may indeed be innocent, but it is also possible that they have technical methods to pass the checking procedure. The third group consists of users who skip or fail the security check, and their accounts are temporarily forbidden from playing video games unless the check is later passed or expires.

Our goal is to train a classifier to identify users that are possibly cheating, based on the previous security check records and other in-game features. The first group can be regarded as reliable positive samples though the number of them is relatively small. For the rest of the groups, we regarde all of them as unlabeled samples, since there is no way for us to know whether they cheated or not for sure. If we conduct PU learning using traditional methods, then the classifiers trained by the data can only identify 'honest cheaters' instead of covering all possible cheaters.
To cope with this problem, we assign soft labels to unlabeled users, on the basis of data obtained from previous check. And we design two methods to calculate soft labels. Details are shown in \ref{sec:details}.

\begin{figure}[htbp]
    \centering
    \includegraphics[width=0.65\linewidth]{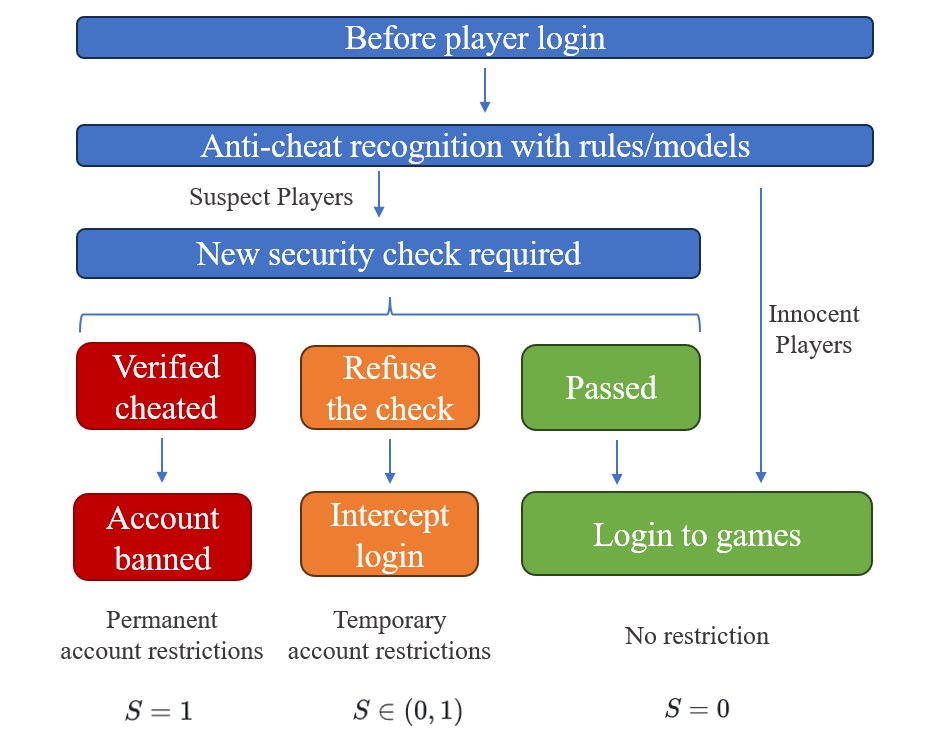}
    \caption{Procedure of game-cheater check and punishment.}
    \label{fig:procedure}
\end{figure}

We then evaluate the performance of soft label PU learning online. To validate the effectiveness of soft labels, we compare the performance of our new method with a baseline method, in which all unlabeled data are simply regarded as negative samples without soft labels. We use the same three-layer neural network as the classifier in both our new method and the baseline method. After that, we ask users with predicted scores higher than a certain threshold to pass security check. The performances are evaluated under two metrics: surrender ratio, defined as the ratio of the number of users who admit themselves as cheaters to the total number of security checks assigned, and the ratio of passing the check, defined as the ratio of the number of passed users to the total number of security checks assigned. A higher surrender ratio and a lower pass ratio indicate better performance. The results are shown in Table 3.

\begin{table}[htbp]
\centering
\begin{threeparttable}
\caption{Results of soft-label PU learning in the anti-cheat task in Tencent Games.}
\label{tab:anticheat}
\small
\begin{tabular*}{0.85\linewidth}{@{\extracolsep{\fill}}lccc@{}}
\toprule
Metric & Method 1 & Method 2 & Baseline \\
\midrule
Surrender ratio & $2.66\%$ & $2.20\%$ & $1.84\%$ \\
Pass ratio      & $19.7\%$ & $24.3\%$ & $53.3\%$ \\
\bottomrule
\end{tabular*}
\end{threeparttable}
\end{table}

The result in Table 3 shows that under both metrics, with both methods for soft label generation, the soft label PU learning method is substantially better than the baseline.

\section{{Discussion}}
In this section, we mainly discuss two questions. To begin with, we discuss the potential application of the soft label PU learning method. Then we discuss a common alternative of soft label PU learning method, i.e. use soft label as features and still conduct normal PU learning.
\subsection{Potential Applications of our proposed soft label PU learning method}
The proposed soft-label PU learning method can potentially be used in a broad range of scenarios involving positive and unlabeled data. Generally speaking, if the number of observed positive samples is not enough, or the labeling of positive samples is obviously uneven, or the labeling negative samples is not absolutely accurate, even unlabeled, our method proposed can be applied. 
Specifically, in section 6, we have discussed the application in the cheater recognition task in Tencent Games. It can mitigate the unlabeled problem in ads recommendation when an ad exposed user disabled the ad tracking leading to a convention unknown observation. Also in medical diagnosis, the observed positive samples are patients that are already diagnosed to have a certain disease. Additionally when our method is used, people who have some related symptoms can be assigned with soft labels depending on their severity. For example, in the diagnosis of Alzheimer's disease ~\citep{jack2008alzheimer}, patients who suffer from mild cognitive impairment (MCI) may be assigned with soft labels.

\subsection{Comparison with using soft label as features}
It may be tempting to use soft label as a feature, and still conduct ordinary PU learning. If the observed positive samples are enough and the labeling is even, then using soft label as features may be better. However, in many applications, labeling is not even and available positive sample size may not be enough. In this case, even if the soft label can be used as a strong feature, the performance of the classifier trained by simple PU method may still not be satisfactory. This has been verified in section 6.1, which shows that after we assign soft labels based on some features that are related to the labels, the AUCs are improved compared with the baseline, which uses these features directly in model training.

\section{Conclusion}
In this paper, we have proposed soft label PU learning, in which apart from labeled positive and ordinary unlabeled data, we have some additional data that are also unlabeled, but are more likely to be positive. In this case, since $T P R$, $F P R, A U C$ are unknown, we have proposed some substitute metrics to evaluate the performance of a classifier. We have shown that under various different assumptions, these metrics can replace the real metrics well. Therefore, model or feature selection can be conducted with the objective of optimizing the PU metrics instead of real metrics. This analysis motivates us to propose a simple parametric learning method to optimize these PU substitute metrics. Finally, we have conducted experiments on both public datasets and tasks in the anti-cheat services of Tencent Games to validate our proposed method.

\newpage

\appendix

This appendix is organized as follows. In \ref{sec:details}, we show the experiment details of soft label generation in Tencent Games. The other sections provide proofs of theorems in the paper.

\section{Experiment details}\label{sec:details}

\subsection{Soft-label generation in the Tencent anti-cheat task}
Figure 2 is a sketch of the procedure of the game-cheater recognition task. In
Tencent games, our goal is to find users who cheat the games in different ways
and let them pass a security check for verification. If the check is not passed or
is refused, various types of punishment will be given to this player account.
For this purpose, we collected some user data avoiding private information
including their gaming actions, complains from their teammates or opponents,
etc. Based on these data, we have enacted multiple rules, which can select users
that are potentially cheaters. The users who satisfy at least one rule are then
required to pass the security check. After that, some users honestly admit or
are recognised by the checking procedure that they are actually cheaters. These
users’ accounts will suffer from permanent time restriction, and we use them as
positive samples. Some users reject or fail to verify themselves. Then a temporary
time restriction will be placed. For these users, we assign them soft labels. Other
users successfully passed the security check. They are labeled 0, indicating that
they are ordinary negative samples and in our cases, they will be regarded as
unlabeled data in case of unknown flaw of the checking procedure. Those players
can play games freely without any restrictions. Finally, The most of users actually
do not satisfy any recognition rules or models. As a result, they are not required
to pass any check procedures, and they are also regarded as ordinary unlabeled
samples.
For users who reject or fail to verify themselves, soft labels are generated
using the following two methods.

\subsection{Soft label based on rules}

The first method is based on the ratio of users that reject or fail to pass the
security check among all users who are required to verify themselves. If such ratio
is high, then the rule is likely to be efficient, and the corresponding users are
more likely to be cheaters. To be more precise, we make the following derivation,
in which $R$ denotes a specific rule, $R_0$ denotes a special rule which selects users
randomly, $F$ denotes the event of rejecting or failing to verify:
\begin{equation}
\begin{aligned}
\mathbb{P}(Y=1\mid F, R)
&= \frac{\mathbb{P}(Y=1, F\mid R)}{\mathbb{P}(F\mid R)}\\
&= \frac{\mathbb{P}(F\mid R)-\mathbb{P}(Y=0,F\mid R)}{\mathbb{P}(F\mid R)}\\
&\ge 1-\frac{\mathbb{P}(F\mid Y=0, R)}{\mathbb{P}(F\mid R)}\\
&\overset{(a)}{=} 1-\frac{\mathbb{P}(F\mid Y=0, R_0)}{\mathbb{P}(F\mid R)}\\
&\overset{(b)}{\ge} 1-\frac{\mathbb{P}(F\mid Y\sim \mathbb{B}_{\pi})}{\mathbb{P}(F\mid R)}\\
&\approx 1-\frac{\text{Ratio of verification failure under random rule}}
{\text{Ratio of verification failure under }R}
\end{aligned}
\label{eq:a1}
\end{equation}

In (a), we assume that an innocent player rejects or fails to verify himself with
equal probabilities for different rules. As a result, the probability of verification
failure under each rule equals such probability under a random rule, that is,
$P(F | Y = 0,R) = P (F | Y = 0,R_0)$. In (b), B denotes Bernoulli distribution,
and $\pi = P(Y = 1)$ is the positive class prior. We assume that a true cheater is
more likely to fail to verify than innocent players.

For every user that rejects or fails security check, the soft label can be assigned as
\begin{eqnarray}
	S=1-\frac{\text{Ratio of verification failure under random rule}}{\text{Ratio of verification failure under }R},
\end{eqnarray}
then for each soft label $s$,
\begin{eqnarray}
	\text{P}(Y=1|S=s)\geq s.
\end{eqnarray}

If all soft labels are larger than $\pi$, then Assumption \ref{ass:basic} is satisfied. The proof is shown as follows. From Bayes formula,
\begin{eqnarray}
	\text{P}(S=s|Y=1)\geq \frac{(1-\pi)s}{\pi(1-s)}\text{P}(S=s|Y=0).
\end{eqnarray}
With the condition that $s\notin (0,\pi)$, we have
\begin{eqnarray}
	\mathbb{E}[S|Y=1]&=&\sum_s s\mathbb{P}(S=s|Y=1)\nonumber\\
	&>&\sum_s s\mathbb{P}(S=s|Y=0)\nonumber\\
	&=&\mathbb{E}[S|Y=0].
\end{eqnarray}
Therefore, Assumption \ref{ass:basic} is satisfied, and thus the soft label PU learning method can be used here.

\subsection{Soft label based on personal security check records}
For each user, we collect the security check record within a fixed time period, such
as 30 days. Then we let $n_i$ be the number of days such that this user is required
for security check, and $k_i$ be the number of days such that this user passed the
security check, $i = 1, . . . ,N$, $N$ is the number of all users. The intuition is that
if $n_i$ is large and $k_i$ is small, then the risk of this user being a cheater is high and
we need to assign this user a high soft label. The precise rule of soft label is:
\begin{eqnarray}
	S_i=1-\hat{\theta}_{Bayes, i},
\end{eqnarray}
in which $\hat{\theta}_{Bayes, i}$ is the Bayes estimation of the probability such that a user successfully verify themselves as innocent players, which is expressed as
\begin{eqnarray}
	\hat{\theta}_{Bayes, i} &=&\int \theta f(\theta|k_i) d\theta\nonumber\\
	&=& \frac{\int\theta f(\theta)\mathbb{P}(k_i|\theta)d\theta}{\int f(\theta)\mathbb{P}(k_i|\theta)d\theta}\nonumber\\
	&=&\frac{\int f(\theta)\theta^{k_i+1}(1-\theta)^{n_i-k_i}d\theta}{\int f(\theta) \theta^{k_i}(1-\theta)^{n_i-k_i}d\theta}.
\end{eqnarray}
In the last step, we simply assume that a user passes security check independently with same probability, thus $k_i$ follows Binomial distribution with parameter $(n_i, \theta)$.

Now it remains to determine the prior distribution $f(\theta)$. We estimate the prior distribution from data. In particular, we define the log likelihood as 
\begin{eqnarray}
	L(\theta) = -\frac{1}{N}\sum_{i=1}^N \ln \int \binom{n_i}{k_i}\theta^{k_i}(1-\theta)^{n_i-k_i}f(\theta)d\theta.
\end{eqnarray}
Now we get $f(\theta)$ by solving the following optimization problem.
\begin{eqnarray}
	\begin{array}{cc}
		\underset{f}{\text{minimize}} & L(\theta)+\lambda\int f^2(\theta) d\theta\\
		\text{subject to} &\int f(\theta)d\theta = 1, f(\theta)\geq 0.
	\end{array}
\end{eqnarray}
$\lambda\int f^2(\theta)d\theta$ is a regularization term that prevents overfitting. The optimization is constrained on a probability simplex, hence we can solve it using exponentiated gradient method \cite{beck2003mirror}.

\section{Proof of Theorem \ref{thm:auc}}\label{sec:auc}

For simplicity, we only provide the proof assuming that $S$ follows a continuous distribution. If the distribution of $S$ have some mass at some $s$, then we can construct a series of continuous distributions that converges to the true distribution of $S$. Let $f_{S}, F_{S}$ be the pdf and cdf of $S$, respectively.

From (3) and (4),
\begin{equation}
	T P R_{S P U}=\frac{\mathbb{P}(\hat{Y}=1)}{\mathbb{E}[S]}-\frac{\mathbb{E}[1-S]}{\mathbb{E}[S]} F P R_{S P U}
\end{equation}

therefore
\begin{eqnarray}
	A U C_{S P U}&= & \frac{1}{\mathbb{E}[S]} \int_{0}^{1} \mathbb{P}(\hat{Y}=1) d F P R_{S P U} \nonumber\\
	&& -\frac{\mathbb{E}[1-S]}{\mathbb{E}[S]} \int_{0}^{1} F P R_{S P U} d F P R_{S P U} \nonumber\\
	&= & \frac{1}{\mathbb{E}[S]}-\frac{\mathbb{E}[1-S]}{2 \mathbb{E}[S]}-\frac{1}{\mathbb{E}[S]} \int_{0}^{1} F P R_{S P U} d \mathbb{P}(\hat{Y}=1) \nonumber\\
	&= & \frac{1}{2}+\frac{1}{2 \mathbb{E}[S]}-\frac{1}{2 \mathbb{E}[S] \mathbb{E}[1-S]}\nonumber \\
	&& +\frac{1}{\mathbb{E}[S] \mathbb{E}[1-S]} \int_{0}^{1} \mathbb{E}[S \hat{Y}] d \mathbb{P}(\hat{Y}=1).
	\label{eq:aucexp}
\end{eqnarray}

Define $T$ implicitly, such that $\mathbb{P}(S>T)=\mathbb{P}(\hat{Y}=1)$, then
\begin{eqnarray}
	\mathbb{E}[S \hat{Y}]&= & \mathbb{E}[S \mathbb{P}(\hat{Y}=1 \mid S)] \nonumber\\
	&= & \mathbb{E}[S I(S>T)]+\mathbb{E}[S(\mathbb{P}(\hat{Y}=1 \mid S)-I(S>T))] \nonumber\\
	&\overset{(a)}{=} & \mathbb{E}[S I(S>T)]  +\mathbb{E}[(S-T)(\mathbb{P}(\hat{Y}=1 \mid S)-I(S>T))] \nonumber\\
	&\overset{(b)}{\leq} & \mathbb{E}[S I(S>T)].	
\end{eqnarray}
(a) holds because $\mathbb{E}[\mathbb{P}(\hat{Y}=1|S)-I(S>T)]=\mathbb{P}(\hat{Y}=1)-\mathbb{P}(S>T)=0$. (b) holds because $(S-T)(\mathbb{P}(\hat{Y}=1 \mid S)-I(S>T))\leq 0$ always holds.

Therefore
\begin{eqnarray}
	\int_{0}^{1} \mathbb{E}[S \hat{Y}] d \mathbb{P}(\hat{Y}=1) & \leq& \int_{0}^{1} \mathbb{E}[S I(S>T)] d \mathbb{P}(S>T) \nonumber\\
	& =&\frac{1}{2}-\frac{1}{2} \int_{0}^{1} F_{S}^{2}(u) du.
	\label{eq:esy}
\end{eqnarray}

Combine \eqref{eq:aucexp} and \eqref{eq:esy}, \eqref{eq:aucbound} can be easily derived. The proof is complete.

\section{Proof of Theorem \ref{thm:linear}}\label{sec:linear}

 For simplicity, we only prove \eqref{eq:tprspu} and \eqref{eq:aucspu}. The proof of \eqref{eq:fprspu} is omitted since it is very similar to the proof of \eqref{eq:tprspu}. From Assumption \ref{ass:scar}, we have $\mathbb{E}[S \mid Y, \hat{Y}]=\mathbb{E}[S \mid Y]$ since $\hat{Y}$ is a function of $X$. Therefore
\begin{eqnarray}
\mathbb{E}[S \hat{Y}] & =&\mathbb{E}[S Y \hat{Y}]+\mathbb{E}[S(1-Y) \hat{Y}] \nonumber\\
& =&\mathbb{E}[Y \hat{Y}] S_{P}+\mathbb{E}[(1-Y) \hat{Y}] S_{N} \nonumber\\
& =&\pi S_{P} T P R+(1-\pi) S_{N} F P R	
\end{eqnarray}
Moreover,
\begin{equation}
	\mathbb{E}[S]=\pi S_{P}+(1-\pi) S_{N},
\end{equation}
hence
\begin{equation}
	T P R_{S P U}=a T P R+b F P R.
\end{equation}

(17) is proved. (18) can be proved in a similar way, thus we omit it. For (19), we have

\begin{align}
AUC_{SPU}
&= \int TPR_{SPU}\, dFPR_{SPU} \notag\\
&= ac \int TPR\, dTPR
   + bd \int FPR\, dFPR \notag\\
&\quad + ad \int TPR\, dFPR
   + bc \int FPR\, dTPR \notag\\
&= \frac{1}{2}ac + \frac{1}{2}bd + bc + (ad-bc)\,AUC \notag\\
&= \frac{1}{2}(b+c) + (ad-bc)\,AUC.
\label{eq:auc_spu}
\end{align}

in which the last step can be simply proved using the expressions from \eqref{eq:a} to \eqref{eq:d}. Furthermore, it is obvious that if $S_{P}>S_{N}$, then $a d-b c>0$.

\section{Proof of Theorem \ref{thm:mela}}\label{sec:mela}

 Let $\hat{Y}: \mathbb{R}^{d} \rightarrow\{0,1\}$ be a classifier, which is a function of $x$. Define $\eta_{S}(x)=\mathbb{E}[S \mid x]$. We first show the following lemma:
\begin{lem}\label{lem}
	Given $\mathbb{P}(\hat{Y}=1)$, if $\hat{Y}=I\left(\eta_{S}(X)>T_{S}\right)$, then $T P R_{S P U}$ is maximized, and $FPR_{S P U}$ is minimized, in which $T_{S}$ is selected such that $\mathbb{P}(\hat{Y}=1)=\mathbb{P}\left(S>T_{S}\right)$.
\end{lem}
\begin{proof}
	\begin{eqnarray}
\mathbb{E}[S \hat{Y}]&= & \mathbb{E}\left[\eta_{S}(X) \hat{Y}\right] \nonumber\\
&= & \mathbb{E}\left[\eta_{S}(X) I\left(\eta_{S}(X)>T_{S}\right)\right]  +\mathbb{E}\left[\eta_{S}(X)\left(\hat{Y}-I\left(\eta_{S}(X)>T_{S}\right)\right)\right] \nonumber\\
&\leq & \mathbb{E}\left[\eta_{S}(X) I\left(\eta_{S}(X)>T_{S}\right)\right].		
	\end{eqnarray}
In the last step, equality holds when $\hat{Y}=I\left(\eta_{S}(X)-T_{S}\right)$. From the definition of $T P R_{S P U}$ and $F P R_{S P U}$ in \eqref{eq:tprdef} and \eqref{eq:fprdef}, we know that with $\hat{Y}=I\left(\eta_{S}(X)-T_{S}\right), T P R_{S P U}$ is maximized, and $F P R_{S P U}$ is minimized.	
\end{proof}

Therefore, if $g_{S}^{*}$ attains optimal $R O C_{S P U}$, then $g_{S}^{*}$ should be monotonically increasing with $\eta_{S}(x)$. From Assumption \ref{ass:mela}, $g_{S}^{*}$ is also monotonically increasing with $\eta(\mathbf{x})$. Similar to Lemma \ref{lem}, it can be proved that $\hat{Y}=I\left(\eta(X)>T_{Y}\right)$ both maximizes $T P R$ and minimizes $F P R$, in which $\eta(X)=$ $\mathbb{E}[Y \mid X]$. Therefore, $g_{S}^{*}$ also attains optimal $R O C$.

\section{Proof of Theorem \ref{thm:nmela}}\label{sec:nmela}
From the proof of Theorem 3, we know that $g^{*}(x)=\eta(x)$ attains optimal $R O C$, while $g_{S}^{*}(x)=\eta_{S}^{*}(x)$ attains optimal $R O C_{S P U}$. For each threshold $T_{Y} \in[0,1]$, define $T P R\left(g, T_{Y}\right)$ as the $T P R$ of classifier $\hat{Y}=I\left(\eta(x)>T_{Y}\right)$, and select $T_{S}$ such that

\begin{equation}
	\int_{\eta_{S}(x) \geq T_{S}} f(x) d x=\int_{\eta(x) \geq T_{Y}} f(x) d x .
\end{equation}

Then define $T P R\left(g_{S}, T_{S}\right)$ as the $T P R$ value of classifier $\hat{Y}=I\left(\eta_{S}(x)>T_{S}\right)$. Then

\begin{equation}
	\begin{aligned}
		T P R\left(g_{S}, T_{S}\right) & =\frac{\mathbb{E}\left[Y I\left(g_{S}(X)>T_{S}\right)\right]}{\mathbb{E}[Y]} \\
		& =\frac{1}{\pi} \int_{g_{S}(x)>T_{S}} \eta(x) f(x) d x \\
		T P R\left(g, T_{Y}\right)= & \frac{1}{\pi} \int_{g(X)>T_{Y}} \eta(x) f(x) d x
	\end{aligned}
\end{equation}

Define

\begin{equation}
	\begin{aligned}
		A & :=\left\{x \mid \eta(x)>T_{Y}\right\} \\
		A_{S} & :=\left\{x \mid \eta_{S}(x)>T_{S}\right\}
	\end{aligned}
\end{equation}

Now we bound the difference of the $T P R$ of classifier $I\left(\eta(x)>T_{Y}\right)$ and $I\left(\eta_{S}(x)>T_{S}\right)$ :

\begin{align}
&TPR\left(g, T_{Y}\right)-TPR\left(g_{S}, T_{S}\right) \notag\\
&= \frac{1}{\pi}\left[
\int_{A \setminus A_{S}} \eta(x) f(x)\, dx
-\int_{A_{S} \setminus A} \eta(x) f(x)\, dx
\right] \notag\\
&= \frac{1}{\pi}\left[
\int_{A \setminus A_{S}}\left(\eta(x)-T_{Y}\right) f(x)\, dx
-\int_{A_{S} \setminus A}\left(T_{Y}-\eta(x)\right) f(x)\, dx
\right] \notag\\
&\le \frac{1}{\pi C_{h}}\left[
\int_{A \setminus A_{S}}\left(h(\eta(x))-h\left(T_{Y}\right)\right) f(x)\, dx
\right. \notag\\
&\qquad\qquad\left.
-\int_{A_{S} \setminus A}\left(h\left(T_{Y}\right)-h(\eta(x))\right) f(x)\, dx
\right] \notag\\
&= \frac{1}{\pi C_{h}}\left[
\int_{A \setminus A_{S}} h(\eta(x)) f(x)\, dx
-\int_{A_{S} \setminus A} h(\eta(x)) f(x)\, dx
\right] \notag\\
&\le \frac{1}{\pi C_{h}}\left[
\int_{A \setminus A_{S}}\left(h(\eta(x))-\eta_{S}(x)\right) f(x)\, dx
\right. \notag\\
&\qquad\qquad\left.
-\int_{A_{S} \setminus A}\left(h(\eta(x))-\eta_{S}(x)\right) f(x)\, dx
\right] \notag\\
&\le \frac{2 \epsilon}{\pi C_{h}} \int_{A \setminus A_{S}} f(x)\, dx.
\label{eq:whatever}
\end{align}

in which the last step comes from Assumption 4. For $x \in$ $A \setminus A_{S}, \eta_{S}(x) \leq T_{S}$, then from Assumption 4,

\begin{equation}
	h(\eta(x)) \leq \eta_{S}(x)+\epsilon \leq T_{S}+\epsilon
\end{equation}

Moreover, for $x \in A_{S} \setminus A, \eta(x) \leq T_{Y}$, we have

\begin{equation}
	T_{S} \leq \eta_{S}(x) \leq h(\eta(x))+\epsilon \leq h\left(T_{Y}\right)+\epsilon
\end{equation}

thus $h(\eta(x)) \leq h\left(T_{Y}\right)+2 \epsilon$, and $\eta(x)-T_{Y} \leq 2 \epsilon / C_{h}$. Hence
\begin{eqnarray}
		 TPR\left(g, T_{Y}\right)-T P R\left(g_{S}, T_{S}\right) 
&\leq& \frac{2 \epsilon}{\pi C_{h}} \int_{\eta(x) \in\left[T_{Y}, T_{Y}+2 \epsilon / C_{h}\right]} f(x) d x \nonumber\\
&\leq & \frac{4 M}{\pi C_{h}^{2}} \epsilon^{2},
\label{eq:tprbound}	
\end{eqnarray}

in which the last step comes from \eqref{eq:slice}. Similarly, it can be proved that

\begin{equation}
	F P R\left(g_{S}, T_{S}\right)-F P R\left(g, T_{Y}\right) \leq \frac{4 M}{\pi C_{h}^{2}} \epsilon^{2}.
	\label{eq:fprbound}
\end{equation}

From \eqref{eq:tprbound} and \eqref{eq:fprbound}, it can be shown that

\begin{equation}
	A U C\left(g_{S}^{*}\right) \geq A U C\left(g^{*}\right)-\frac{8 M}{\pi C_{h}^{2}} \epsilon^{2} .
\end{equation}

\section{Proof of Theorem \ref{thm:roccurve}}\label{sec:roccurve}
Theorem \ref{thm:roccurve} has been proved in Lemma \ref{lem}.
\vskip 0.2in

\bibliographystyle{elsarticle-num}
\bibliography{sample-base}

\end{document}